\documentclass[preprint,3p,12pt]{elsarticle}
\makeatletter
\def\ps@pprintTitle{%
 \let\@oddhead\@empty
 \let\@evenhead\@empty
 \def\@oddfoot{\centerline{\thepage}}%
 \let\@evenfoot\@oddfoot}
\makeatother
\usepackage{soul}
\usepackage{graphics}
\usepackage{psfig}
\usepackage{amsmath}
\usepackage{rotating}
\usepackage[absolute]{textpos}
\usepackage{color}
\usepackage{lineno}
\usepackage[utf8]{inputenc}
\usepackage{amsmath}
\usepackage{amsfonts}
\usepackage{amssymb}
\usepackage{amsthm}
\usepackage{mathrsfs}
\usepackage{latexsym}
\usepackage{multirow}
\usepackage{rotating}
\usepackage{caption}
\usepackage{graphicx}
\usepackage{float}
\usepackage{array}
\usepackage{subfigure}
\usepackage{tabularx}
\usepackage{graphicx,psfig}
\usepackage{footnote}
\usepackage[titletoc,toc,title]{appendix}
\usepackage{colortbl}
\usepackage[dvipsnames]{xcolor}
\usepackage{hyperref}
\usepackage{natbib}
\usepackage{amssymb}
\usepackage{soul}
\usepackage{algorithmic}
\usepackage{algorithm}
\usepackage{epsfig}
\usepackage{relsize}
\usepackage{color}
\usepackage{mathtools}
\usepackage{multirow}
\usepackage[utf8]{inputenc}
\usepackage[english]{babel}

\newtheorem{theorem}{Theorem}

\newtheorem{prop}{Proposition}
\newtheorem*{remark}{Remark}

\begin{document}

\begin{frontmatter}

\title{Imbalanced Ensemble Classifier for learning from imbalanced business school data set}
\author{Tanujit Chakraborty\footnote{\textit{Email}: Tanujit Chakraborty (tanujit\_r@isical.ac.in)}\\
    {\scriptsize \textsuperscript{} SQC and OR Unit, Indian Statistical Institute, 203, B. T. Road, Kolkata - 700108,
    India}\\}

\begin{abstract}
Private business schools in India face a common problem of selecting
quality students for their MBA programs to achieve the desired
placement percentage. Generally, such data sets are biased towards
one class, i.e., imbalanced in nature. And learning from the
imbalanced dataset is a difficult proposition. This paper proposes
an imbalanced ensemble classifier which can handle the imbalanced
nature of the dataset and achieves higher accuracy in case of the
feature selection (selection of important characteristics of
students) cum classification problem (prediction of placements based
on the students' characteristics) for Indian business school
dataset. The optimal value of an important model parameter is found.
Numerical evidence is also provided using Indian business school
dataset to assess the outstanding performance of the proposed
classifier.
\end{abstract}


\end{frontmatter}

\section{Introduction}

Out of the many reasons behind the closing down of many of the
private business schools, the foremost one is the unemployment of
Master of Business Administration (MBA) students passing out of
these business schools. The most challenging job for administrations
is to find the optimal set of parameters for choosing the right
candidates in their MBA program which will ensure the employability
of the candidates. Attracting students in business schools are
highly dependent on the schools' past placement records. If the
right set of students are not selected for a few years, the number
of unplaced students will certainly accumulate, resulting in the
damage of reputation for the business school. One needs to develop a
model in such a way that the model ensures appropriate feature
selection (selection of important student's characteristics) with a
decision on the optimal values or ranges of the features and higher
prediction accuracy of the classifier as well. In our previous
works, we proposed a hybrid classifier based on classification tree
(CT) and artificial neural network (ANN) (to be referred to as
hybrid CT-ANN model in the rest of the paper) to solve the business
school problem \citep{chakrabortynovel}. In this article, we
identified a vital property of the business school data set, i.e.,
its imbalanced nature. Usual classifiers make a simple assumption
that the classes to be distinguished should have a comparable number
of instances. Many real-world data sets including the business
school dataset are skewed, in which many of the cases belong to a
larger class and fewer cases belong to a smaller, yet usually more
interesting class. There are also the cases where the cost of
misclassifying minority examples is much higher in terms of the
seriousness of the problem in hand \citep{rastgoo2016tackling}. Due
to higher weightage are given to the majority class, these systems
tend to misclassify the minority class examples as the majority, and
lead to a high false negative rate. In this particular example of
business school data set, it is clearly a two-class problem with the
class distribution of 80:20, where a straightforward method of
guessing all instances to be placed would achieve an accuracy of
80\%.

There are broadly two ways to deal with imbalanced data problems.
One such way to deal with the imbalanced data problems is to modify
the class distributions in the training data by applying sampling
techniques. Sampling techniques include oversampling the minority
class to match the size of the majority class and/or undersampling
the majority class to match the size of the minority class. Sampling
is a popular strategy to handle the data imbalance as it simply
rebalances the data at the data preprocessing stage. But these
approaches have obvious deficiencies like undersampling majority
instances may lose potential useful information of the data set and
oversampling increases the size of the training data set, which may
increase computational cost. Nonetheless, sampling is not the only
way for handling imbalanced data sets. There exist some specially
designed ``imbalanced data-oriented" algorithms which perform well
on unmodified original imbalanced data sets. One of the most
celebrated paper in the literature is hellinger distance decision
tree (HDDT) \citep{cieslak2008learning} which uses hellinger
distance (HD) as a decision tree splitting criterion and it is
insensitive towards the skewness of the class distribution
\citep{cieslak2012hellinger}. An immediate extension to this work is
HD based random forest (HDRF) \citep{su2015improving}. Another
breakthrough in the literature is the class confidence proportion
decision tree (CCPDT), a robust decision tree algorithm which can
also handle original imbalanced datasets \citep{liu2010robust}. It
is to be noted that ``imbalanced data-oriented" classifiers are
sometimes preferred since they work with original data sets. We are
therefore motivated to ask: Can we create an ensemble imbalanced
data-oriented classifier which can improve the performance of HDDT,
mitigate the need of sampling and solve an Indian business school
data problem?

In response to this question, we proposed an ensemble classifier for
feature selection cum classification problems which can be used to
solve the imbalanced business school dataset problem. Our proposed
ensemble classifier has the advantages of both the HDDT and ANN
algorithm and performs well in high dimensional feature spaces. The
optimal choice of an important model parameter is also proposed in
this paper. Further numerical evidence based on business school
dataset shows the robustness of the proposed algorithm.

This paper is organized as follows. In section 2, we describe the
proposed ensemble model. The theoretical results are presented in
section 3 and experimental evaluation is shown in section 4. Section
5 is fully devoted to the concluding remarks of the paper.

\section{Methodology}

\subsection{\textbf{An overview on HDDT}\\}

Chawla \citep{cieslak2008learning} proposed HDDT which uses HD as
the splitting criterion to build a decision tree. HD is used as a
measure of distributional divergence and has the property of skew
insensitivity \citep{rao1995review}. Let $(\Theta, \lambda)$ denote
a measurable space. For any binary classification problem, let us
suppose that $P$ and $Q$ be two continuous distributions with
respect to the parameter $\lambda$ having the densities $p$ and $q$
in a continuous space $\Omega$, respectively. Define HD as follows:
\[
d_{H}(P,Q)=\sqrt
{\int_{\Omega}(\sqrt{p}-\sqrt{q})^{2}d\lambda}=\sqrt
{2\bigg(1-\int_{\Omega}\sqrt{pq}d\lambda\bigg)}
\]
where $\int_{\Omega}\sqrt{pq}d\lambda$ is the Hellinger integral. It
is noted that HD doesn't depend on the choice of the parameter
$\lambda$. Given a countable space $\Phi$, HD can also be written as
follows:
\[
d_{H}(P,Q)=\sqrt{\sum_{\phi\in\Phi}\bigg(\sqrt{P(\phi)}-\sqrt{Q(\phi)}\bigg)^{2}}
\]
The bigger the value of HD, the better is the discrimination between
the features. A feature is selected that carries the minimal
affinity between the classes. For the application of HD as a
decision tree criterion, the final formulation can be given as
follows:
\begin{equation}
d_{H}(X_{+},X_{-})=\sqrt{\sum_{j=1}^{K}\bigg(\frac{|X_{+j}|}{|X_{+}|}-\frac{|X_{-j}|}{|X_{-}|}\bigg)^{2}}
\end{equation}
where $|X_{+}|$ indicates the number of examples that belong to the
majority class in training set and $|X_{+j}|$ is the subset of the
training set with the majority class and the value $j$ for the
feature $X$. A similar explanation can be written for $|X_{-}|$ and
$|X_{-j}|$ but for the minority class. Here $K$ is the number of
partitions of the feature space $X$. Since equation (6) is not
influenced by prior probability, it is insensitive to the class
distribution. Based on the experimental results, Chawla
\citep{cieslak2008learning} concluded that unpruned HDDT is
recommended for dealing with imbalanced problems as a better
alternative to sampling approaches.\\

\subsection{\textbf{An overview on ANN}\\}

Neural network models are inspired by biological nervous systems
\cite{fukushima1982neocognitron}. The network functions are
determined largely by the connections between elements. The train of
a neural network can be done by performing a particular function by
adjusting the values of the connections (weights) between elements.
Neural networks are trained so that a particular input (feature
vectors) leads to a specific target output (class level). The
network is adjusted, based on a comparison of the output and the
target, until the network output matches the predicted class.
Mapping function used in ANN is very flexible. Given the right
weights, this function can approximate almost any functional form to
any degree of accuracy. This function approximation is mainly done
by an activation function (for example, sigmoid, logsig, tansig,
etc). A common neural network architecture is shown in Fig. 1.

\begin{figure}[H]
\centering
\includegraphics[scale=0.50]{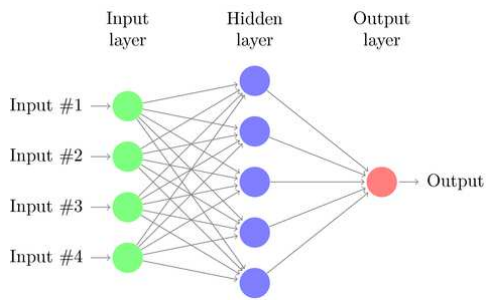}
\caption{An example of artificial neural network with one hidden
layer}
\end{figure}

While training the network with any particular dataset, the problem
of overfitting can be avoided by training the network for a limited
number of epochs \cite{yegnanarayana2009artificial}. Standard
backpropagation (feedforward) is a gradient descent algorithm where
the weights are moved along the negative of the gradient of the
performance function. Typically, a new input leads to an output
similar to the correct outputs if it is properly trained for input
vectors used in training. Complex neural networks have more than one
hidden layers in its architecture.\\

\subsection{\textbf{Proposed Imbalanced Ensemble Classifier}\\}

The motivation behind designing an ensemble classifier for
imbalanced data sets is that one we would like to work with the
original data set without taking recourse to sampling. Here we are
going to create an ensemble classifier which will utilize the power
of HDDT as well as the superiority of neural networks. In the
proposed imbalanced ensemble classifier (to be denoted by IEC in the
rest of the paper), we first split the feature space into areas by
HDDT algorithm. Most important features are chosen using HDDT and
redundant features are extracted. We then build a ANN model using
the important variables obtained through HDDT algorithm. Also, the
prediction results obtained from HDDT are used as another input
information in the input layer of neural networks. The effectiveness
of the proposed classifier lies in the selection of important
features and using prediction results of HDDT followed by the ANN
model. The inclusion of HDDT output as an additional input feature
not only improves the model accuracy but also increases class
separability. The informal workflow of our proposed IEC model, shown
in Figure 2 is as follows:

\begin{figure}[t]
\centering
\includegraphics[scale=0.48]{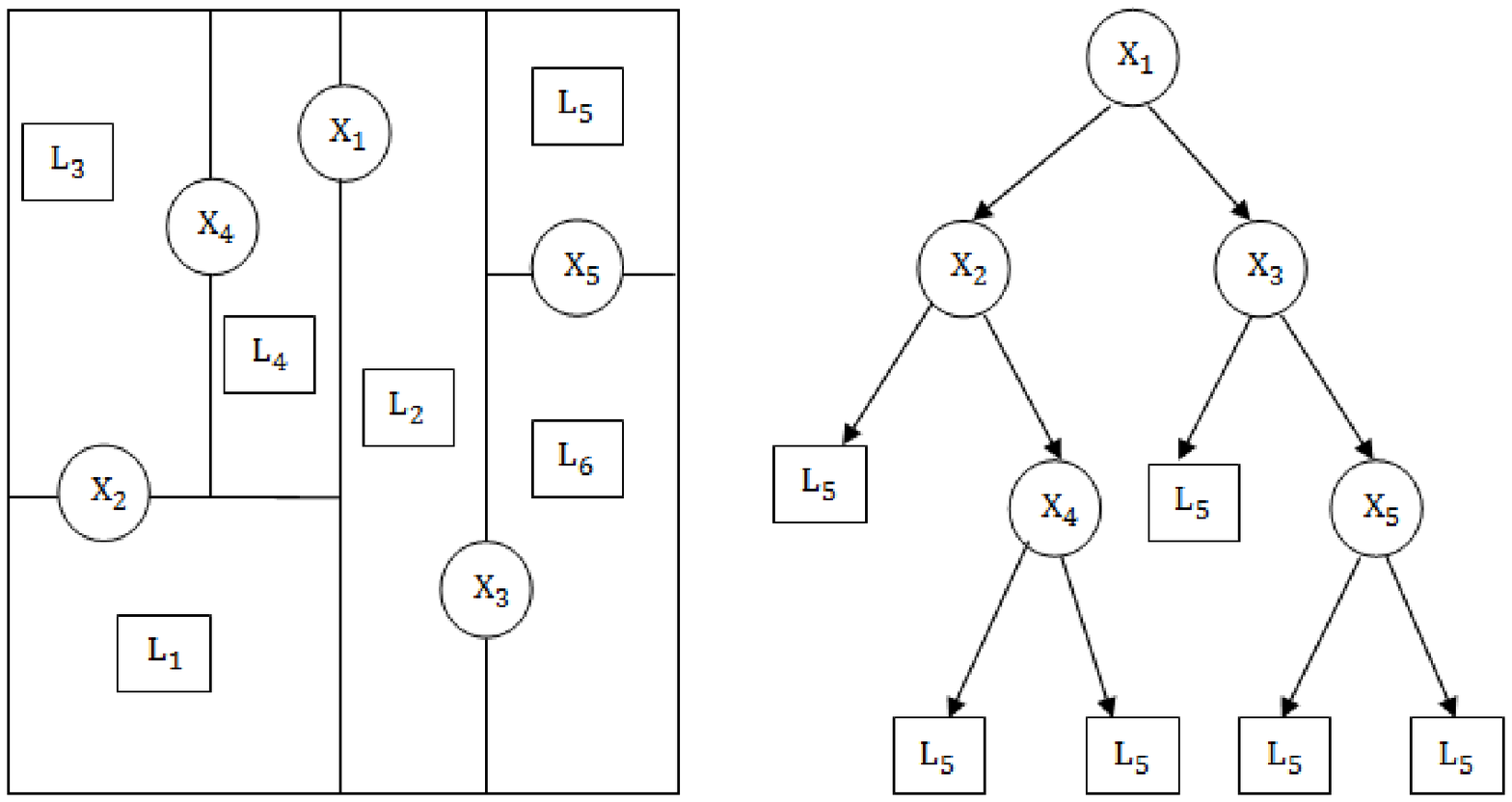}
\includegraphics[scale=0.48]{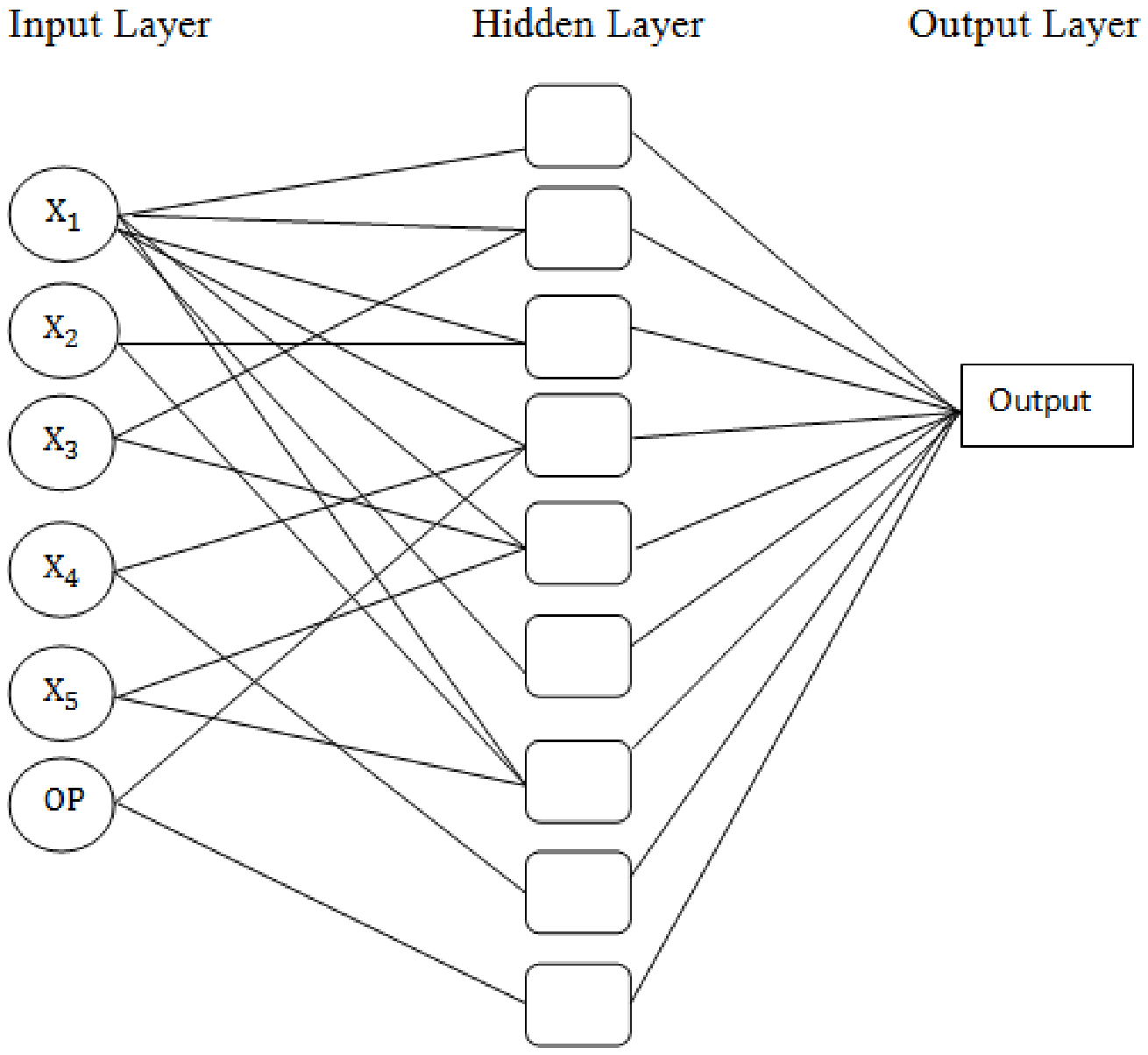}
\caption{An example of superensemble classifier with $X_{i};
i=1,2,3,4,5$ as Important features obtained by HDDT, $L_{i}$ as leaf
nodes and OP as HDDT output.} \label{figmodRWNGT}
\end{figure}

\begin{itemize}
\item Sort the feature value in ascending order and find the splits
between adjacent different values of the feature. Calculate the
binary conditional probability divergence at each split using HD
measure (see equation (1)).
\item Record the highest divergence as the divergence of the whole
feature. Choose the feature that has maximum HD value and grow
unpruned HDDT.
\item Using the HDDT algorithm, build a decision tree. Feature selection model generated by HDDT takes into account the
imbalanced nature of the data set.
\item The prediction result of HDDT algorithm is used as an additional feature in
the input layer of the ANN model. Export important input variables
along with additional feature to the ANN model and a neural network
is generated.
\item Since the output results of HDDT has been incorporated as an
additional feature along with other important features obtained by
HDDT in the input layer of ANN, the number of hidden layer is chosen
to be one.
\item A one hidden layered ANN with sigmoid activation function
having number of neurons in the hidden layer to be
$O\bigg(\sqrt\frac{n}{dlogn}\bigg)$, where $n$ is the number of
training samples, $d$ is the number of input features in the ANN
model, is trained (see section 3). And finally record the
classification results.
\end{itemize}

IEC not only handles imbalance through the implementation of HDDT in
selecting features but also improves the performance of the
classifier by incorporating better classification results for the
data set obtained from HDDT and the model gets improved using ANN
algorithm. This algorithm is a two-step problem-solving approach
such as handling imbalanced class distribution, selecting important
features and getting an improved ensemble classifier. The optimal
characteristics of students which affect the placements can be
chosen by our model and future predictions while modeling the
imbalanced dataset can also be done by IEC.

\section{Optimal value of IEC model parameter}

Our proposed IEC has the following architecture: first, it extracts
important features from the feature space using the HDDT algorithm,
then it builds one hidden layered ANN model with the important
features extracted using HDDT along with HDDT outputs as an
additional feature. Now we are going to find out the optimal value
of the number of neurons in the hidden layer of the proposed model.

Let $\underline{X}$ be the space of all possible values of $p$
features and $C$ be the set of all possible binary outcomes. We are
given a training sample with $n$ observations, \\
$L=\{(X_{1},C_{1}), (X_{2},C_{2}),...,(X_{n},C_{n})\}$, where
$X_{i}=(X_{i1},X_{i2},...,X_{ip}) \in \underline{X}$ and $C_{i} \in
C$. We build IEC with HDDT given features and $OP$ as another input
feature in the model. The dimension of the input layer in the ANN
model, to be denoted by $d_{m}(\leq{p})$, is the number of important
features obtained by HDDT + 1. We have used one hidden layer in the
model due to the incorporation of $OP$ as an input information in
the model. It should be noted that one-hidden layered neural
networks yield strong universal consistency and there is little
theoretical gain in considering two or more hidden layered neural
networks \cite{devroye2013probabilistic}. In IEC model, we have used
one hidden layer with $k$ neurons. This makes the proposed ensemble
binary classifier less complex and less time consuming while
implementing the model. After elimination of redundant features by
HDDT and incorporating OP as another input vector, let us now
consider the following training sequence
$\xi_{n}=\{(Z_{1},Y_{1}),...,(Z_{n},Y_{n})\}$ of $n$ i.i.d copies of
$(\underline{Z},\underline{Y})$ taking values from
$\mathbb{R}^{d_{m}}\times C$. A classification rule realized by a
one-hidden layered neural network having logistic sigmoid activation
function is chosen to minimize the empirical $L_{1}$ risk, where the
$L_{1}$ error of a function $\psi : \mathbb{R}^{d_{m}}\rightarrow
\{0,1\}$ is defined by $J(\psi)=E\{ |\psi(Z)-Y| \}$. The theorem
stated below is based on the idea of Lugosi \& Zeger (1995)
\cite{lugosi1995nonparametric} which states the regularity
conditions for universal consistency of the one hidden layered ANN
model.

\begin{theorem}
Consider a neural network with one hidden layer with bounded
output weight having $k$ hidden neurons and let $\sigma$ be a
logistic squasher. Let $\mathscr{F}_{n,k}$ be the class of neural
networks with logistic squasher defined as
\[
\mathscr{F}_{n,k}=\Bigg\{
\sum_{i=1}^{k}c_{i}\sigma(a_{i}^{T}z+b_{i})+c_{0} : k \in
\mathbb{N}, a_{i} \in \mathbb{R}^{d_{m}}, b_{i},c_{i} \in
\mathbb{R}, \sum_{i=0}^{k}|c_{i}|\leq \beta_{n} \Bigg\}
\]
and let $\psi_{n}$ be the function that minimizes the empirical
$L_{1}$ error over $\psi_{n} \in \mathscr{F}_{n,k}$. It can be shown
that if $k$ and $\beta_{n}$ satisfy
\[
k \rightarrow \infty , \quad \beta_{n} \rightarrow \infty , \quad
\frac{k \beta_{n}^{2}log(k \beta_{n})}{n} \rightarrow 0
\]
then the classification rule
\begin{equation}
  g_{n}(z)=\begin{cases}
    0, & \text{if $\psi_{n}(z)\leq 1/2$}.\\
    1, & \text{otherwise}.
  \end{cases}
\end{equation}
is universally consistent.
\end{theorem}
Equivalently, we write $J(\psi_{n})-J^{*} \rightarrow 0$ in
probability, where $J(\psi_{n})=E\{|\psi_{n}(Z)-Y| | \xi_{n}\}$ and
$J^{*}=\inf_{\psi_{n}}J(\psi_{n})$ \cite{devroye2013probabilistic}.
Write
\[
J(\psi_{n})-J^{*} = \bigg(  J(\psi_{n})-\inf_{\psi \in
\mathscr{F}_{n,k}} J(\psi) \bigg)  + \bigg ( \inf_{\psi \in
\mathscr{F}_{n,k}} J(\psi) - J^{*} \bigg )
\]
where, $( J(\psi_{n})-\inf_{\psi \in \mathscr{F}_{n,k}} J(\psi) )$
is called estimation error and $( \inf_{\psi \in \mathscr{F}_{n,k}}
J(\psi) - J^{*} )$ is called approximation error.\\

To obtain the optimal choice of $k$ of the proposed model it is
necessary to obtain the upper bounds on the rate of convergence,
i.e., how fast $J(\psi_{n})$ approaches to zero
\cite{gyorfi2006distribution}. Though in case of the rate of
convergence of estimation error, we will have a distribution-free
upper bound \cite{farago1993strong}. And to obtain the optimal value
of $k$, it is enough to find upper bounds of the estimation and
approximation errors. The upper bound of approximation error
investigated by Baron \cite{barron1993universal}.

\begin{prop}
For a fixed $d_{m}$, let $\psi_{n} \in \mathscr{F}_{n,k}$. If the
proposed model satisfies the regularity conditions of strong
universal consistency as stated in Theorem 1, then the optimal
choice of $k$ is $O\bigg( \sqrt{\frac{n}{d_{m}log(n)}} \bigg)$.
\end{prop}
\begin{proof}

The upper bound of approximation error is found by Baron
\cite{barron1993universal} to be $O \bigg ( \frac{1}{\sqrt{k}}
\bigg)$. \\Though the approximation error goes to zero as the number
of neurons goes to infinity for strongly universally consistent
classifier, for practical implementation the number of neurons is
often fixed (eg., can't be increased with the size of the training
sample). \\ Using lemma 3 of \cite{farago1993strong}, we can write
that the estimation error is always $O \bigg(
\sqrt{\frac{kd_{m}log(n)}{n}}\bigg)$. \\ Bringing the above facts
together, we can write
\[
J(\psi_{n})-J^{*}=O \bigg( \sqrt{\frac{kd_{m}log(n)}{n}} +
\frac{1}{\sqrt{k}} \bigg)
\]
Now, to find optimal value of $k$, the problem reduces to equating
$\sqrt{\frac{kd_{m}log(n)}{n}}$ with $\frac{1}{\sqrt{k}}$, which
gives $k=O\bigg( \sqrt{\frac{n}{d_{m}log(n)}} \bigg)$.
\end{proof}

\begin{remark}
The optimal value of hidden nodes is found to be $O\bigg(
\sqrt{\frac{n}{d_{m}log(n)}} \bigg)$ for the universally consistent
IEC model. For practical use, if the data set is small or
medium-sized, the recommendation is to use the number of hidden
nodes in the hidden layer to be $ \sqrt{\frac{n}{d_{m}log(n)}}$ for
achieving the utmost accuracy of the proposed model. The practical
usefulness and competitiveness of the proposed classifier in solving
a real life imbalanced business school data problem are shown in the
next Section.
\end{remark}

\section{Application to Indian Business School Data}

In this section, we first describe the business school data in brief
and also discuss different evaluation measures that are used in this
study. Subsequently, we are going to report the experimental results
and compare our proposed IEC model with other state-of-the-art
classifiers.

\subsection{\textbf{Description of data set}}

The data was provided by a private business school which receives
applications for the MBA program from across the country and admits
a pre-specified number of students every year. This dataset
comprises several parameters of last $5$ years passed out students'
profile along with their placement information. The dataset has 17
explanatory variables out of which 7 categorical variables and 10
continuous variables which represent the parameters of the students
and one response variable, namely placement which indicates whether
the student got placed or not \cite{chakrabortynovel}. In order to
measure the level of imbalance of these datasets, we compute the
coefficient of variation (CV) which is the proportion of the
deviation in the observed number of examples for each class versus
the expected number of examples in each class \cite{wu2007local}.
The datasets with a CV more than equal to $0.30-$ a class ratio of
$2:1$ on a binary dataset is chosen as imbalanced data. In the
business school dataset, CV turns out to be 0.50. We also applied
$5\times2$ cross-validation while evaluating classifiers on the
datasets \cite{dietterich1998approximate}, in which each dataset is
broken into class-stratified halves, allowing two experiments in
each half, one is used as training ($70\%$ of the data) and others
as testing ($30\%$ of the data). The experiments are repeated 5
times and the average results are reported in the paper. Table 1
gives an overview of these data sets.

\begin{table}[H]
\tiny \centering \caption{Sample business school data set.}
    \begin{tabular}{cccccccccc}
        \hline
        ID  & Gender & SSC  & HSC  & Degree  & E.Test  &  SSC &  HSC &  HSC & Placement          \\
            &        & Percentage & Percentage &  Percentage &  Percentile & Board & Board & Stream  &  \\ \hline
        1  & M      & 68.4   & 85.6    & 72      & 70   & ICSE    & ISC    & Commerce & Y        \\
        2  & M      & 59     & 62      & 50      & 79   & CBSE    & CBSE   & Commerce & Y \\
        3  & M      & 65.9   & 86      & 72      & 66   & Others  & Others & Commerce & Y \\
        4  & F      & 56     & 78      & 62.4    & 50.8 & ICSE    & ISC    & Commerce & Y \\
        5  & F      & 64     & 68      & 61      & 24.3 & Others  & Others & Commerce & N \\
        6  & F      & 70     & 55      & 62      & 89   & Others  & Others & Science  & Y \\
        .  & .      & .      & .       & .       & .    & .       & .      & .        & . \\
        .  & .      & .      & .       & .       & .    & .       & .      & .        & . \\
        .  & .      & .      & .       & .       & .    & .       & .      & .        & . \\
        \hline
    \end{tabular}

\end{table}

\subsection{\textbf{Performance measures}}

The performance evaluation measures used in our experimental
analysis are based on the confusion matrix. Higher the value of
performance metrics, the better the classifier is. The
expressions for different performance measures as follows: \\\\
G-mean = $\sqrt{\mbox{Sensitivity}\times \mbox{Specificity}}$; AUC =
$\frac{\mbox{Sensitivity}+\mbox{Specificity}}{2}$;\\ F-measure =
$2\frac{\big(\mbox{Precision}\times\mbox{Sensitivity}\big)}{\big(\mbox{Precision}+\mbox{Sensitivity}\big)}$;
Accuracy = $\frac{(TP+TN)}{(TP+TN+FP+FN)}$; \\\\
where, Precision = $\frac{TP}{TP+FP}$; Sensitivity =
$\frac{TP}{TP+FN}$; Specificity = $\frac{TN}{FP+TN}$.

\subsection{\textbf{Analysis of results}}

We aim to select the optimal set of features and the corresponding
model for the selection of the right set of students who will be fit
for the MBA program of a business school and subsequently will be
placed as well. We compare our proposed imbalanced ensemble
classifier (IEC) with mostly other similar types of ``imbalanced
data-oriented" classifiers. Different performance metrics are
computed to draw the conclusion from the experimental results. All
the methods were implemented in the R Statistical package on a PC
with 2.1 GHz processor and 8 GB memory. \\ We started the
experimentation with HDDT algorithm by using R Package `CORElearn'
for learning from imbalanced business school data set. HDDT achieved
around 93\% accuracy while CT achieved around 83\% accuracy. This
gives an indication that ``imbalanced data-oriented" classifiers
perform better than the traditional supervised classifiers designed
for general purposes. Further, we implemented HDRF, CCPDT which are
among other imbalanced data-oriented algorithms. Finally, we applied
our proposed imbalanced ensemble classifier which is a two-step
methodology. In the first stage, we select important features using
HDDT and record its classification outputs. Below are the important
features we obtained for business school data set by applying HDDT:
SSC Percentage, HSC Percentage, Entrance Test Percentile, Degree
Percentage, and Work Experience. In the next step, we design a
neural network with the above mentioned important features along
with HDDT output as an additional feature vector. The number of
hidden neurons in the hidden layer of the model is chosen based on
the recommendation of the model (see Remark in Section 3). Min-max
method is used for scaling the data in an interval of $[0,1]$. ANN
training was done using `neuralnet' implementation in R. We reported
the performance of different classifiers in terms of different
performance metrics in Table 2. It is clear from Table 2 that our
proposed methodology achieved an accuracy of 96\% for prediction in
business school data set.

\begin{table}[H]
\centering \caption{Quantitative measure of performance for
different classifiers}
    \begin{tabular}{ccccc}
        \hline
        Classifiers                                 & AUC       & F-measure & G-mean   & Accuracy \\ \hline
        CT                                          & 0.810     & 0.822     & 0.815    & 0.833    \\
        ANN                                         & 0.768     & 0.781     & 0.758    & 0.771    \\
        HDCT                                        & 0.933     & 0.936     & 0.925    & 0.931    \\
        HDRF                                        & 0.939     & 0.941     & 0.932    & 0.938    \\
        CCPDT                                       & 0.912     & 0.918     & 0.902    & 0.915    \\
        \textbf{IEC}           & \textbf{0.964}     & \textbf{0.969}     & \textbf{0.951}    & \textbf{0.960}    \\ \hline
    \end{tabular}

\end{table}

\section{Conclusion}

We proposed an imbalanced ensemble classifier (IEC) which takes into
account data imbalance and used it for feature selection cum
classification problems. Through experimental evaluation, we have
shown our proposed methodology performed well compared to the other
state-of-the-art models. It is also important to note that
``imbalanced data-oriented" algorithms perform well on the original
imbalanced datasets \cite{cieslak2012hellinger}. If we would like to
work with the original data without taking recourse to sampling, our
proposed methodology will be quite handy. IEC has the desired
statistical properties like universal consistency, less tuning
parameters and achieves higher accuracy than HDDT and ANN model. We
thereby conclude that for the imbalanced business school data set it
is sufficient to use IEC model without taking recourse to sampling
or any other imbalanced data-oriented single classifiers. Due to the
robustness of the proposed IEC algorithm, it can also be useful in
other imbalanced classification problems as well.

\bibliographystyle{elsarticle-num}
\bibliography{bibliography}

\end{document}